\documentclass{article}
\pdfoutput=1

\usepackage[utf8]{inputenc}
\usepackage{amssymb, amsmath, amsthm}
\usepackage{algorithm, algpseudocode}
\usepackage{enumitem}
\usepackage{url}
\usepackage[
  pdftex,
  colorlinks,
  linkcolor=red,
  citecolor=blue,
  urlcolor=cyan,
  filecolor=magenta
]{hyperref}

\newtheorem{lemma}{Lemma}[section]
\newtheorem{remark}[lemma]{Remark}
\newtheorem{definition}[lemma]{Definition}
\newtheorem{proposition}[lemma]{Proposition}
\newtheorem{theorem}[lemma]{Theorem}
\newtheorem{corollary}[lemma]{Corollary}
\theoremstyle{remark}

\DeclareMathOperator*{\argmin}{argmin}
\DeclareMathOperator{\sign}{sign}
\newcommand{\real}{\mathbb{R}}
\newcommand{\E}[1]{\mathbb{E}\left[#1\right]}
\newcommand{\pr}[1]{P\left[#1\right]}
\DeclareMathOperator{\srank}{sr}
\newcommand{\scale}[1]{S_{#1}}
\newcommand{\tmpepsilon}{\hat{\epsilon}}
\newcommand{\outX}{\bar{X}}
\newcommand{\outY}{\bar{Y}}
\newcommand{\outZ}{\bar{Z}}
\newcommand{\outq}{\bar{q}}
\newcommand{\outz}{\bar{z}}
\newcommand{\cweights}{\mathcal{W}}
\newcommand{\variance}{\nu}
\newcommand{\textalg}[1]{\textnormal{\textproc{#1}}}
\newcommand{\SparseFactor}{\textalg{SparseFactorization}}

\begin{document}

\title{Non-Convex Compressed Sensing with Training Data}
\author{G. Welper\footnote{Department of Mathematics, University of Central Florida, Orlando, FL 32816, USA, email \href{mailto:gerrit.welper@ucf.edu}{\texttt{gerrit.welper@ucf.edu}}. 
}}

\date{}
\maketitle

\begin{abstract}

  Efficient algorithms for the sparse solution of under-determined linear systems $Ax = b$ are known for matrices $A$ satisfying suitable assumptions like the restricted isometry property (RIP).
  Without such assumptions little is known and without any assumptions on $A$ the problem is $NP$-hard. 
  A common approach is to replace $\ell_1$ by $\ell_p$ minimization for $0 < p < 1$, which is no longer convex and typically requires some form of local initial values for provably convergent algorithms. 

  In this paper, we consider an alternative, where instead of suitable initial values we are provided with extra training problems $Ax = B_l$, $l=1, \dots, p$ that are related to our compressed sensing problem. They allow us to find the solution of the original problem $Ax = b$ with high probability in the range of a one layer linear neural network with comparatively few assumptions on the matrix $A$.

\end{abstract}



\section{Introduction}

We consider the compressed sensing problem to find the sparsest solution of an under-determined linear system: For $A \in \real^{m \times n}$ with $m < n$ and $b \in \real^m$, solve
\begin{equation}
  \begin{aligned}
    & \min_{x \in \real^n} \|x\|_0 & & \text{subject to} & A x & = b,
  \end{aligned}
  \label{eq:cs0}
\end{equation}
where $\|x\|_0$ denotes the number of non-zero entries of $x$. This problem is $NP$-hard in general \cite{Natarajan1995}, but can be solved by the convex optimization problem
\begin{equation*}
  \begin{aligned}
    & \min_{x \in \real^n} \|x\|_1 & & \text{subject to} & A x & = b
  \end{aligned}
\end{equation*}
under additional assumptions. The most common is the \emph{$(s,\epsilon)$-Restricted Isometry property (RIP)}
\begin{equation*}
  \begin{aligned}
    (1-\epsilon) \|x\| & \le \|Ax\| \le (1+\epsilon)\|x\| &
    & \text{for all $s$-sparse }x \in \real^n,
  \end{aligned}
\end{equation*}
with fairly strict requirements $\epsilon < 4/\sqrt{41} \approx 0.6246$ on the RIP parameters \cite{CandesRombergTao2006a,Donoho2006,CandesRombergTao2006,FoucartRauhut2013}. 

However, sparsity is also of interest in many applications that violate the RIP. This is addressed by several papers that aim for weaker assumptions on the sensing matrix. A common approach \cite{CandesWakinBoyd2008,ChartrandStaneva2008,FoucartLai2009,Sun2012,ShenLi2012} considers the optimization of $\ell_p$-norms with $0 < p < 1$, which resemble the $\ell_0$-norm more closely than the $\ell_1$-norm, but renders the optimization problem non-convex and again $NP$-hard in the worst case \cite{GeJiangYe2011}. In addition, several iterative algorithms \cite{CandesWakinBoyd2008,ChartrandWotaoYin2008,FoucartLai2009,DaubechiesDeVoreFornasierEtAl2010,LaiXuYin2013,WoodworthChartrand2016}, mostly variations of reweighted least squares methods, show promising behaviour for non-RIP matrices.

Partial theoretical results are available for the above methods, but since the problems are non-convex, they typically require some form of good initial values. In this paper, we consider an alternative perspective. Instead of approaching the compressed sensing problem \eqref{eq:cs0} directly, we assume that we are provided a set of related training problems $Ax = B_l$, $l = 1, \dots, q$, containing sufficiently many ``easy'' problems in the sense that they can be solved by $\ell_1$-minimization. We show that after training on these simpler problems, we can also solve the compressed sensing problem for difficult right hand sides $b$, for which $\ell_1$-minimization fails and currently no alternative algorithms with guaranteed sparse recovery are available. The approach is loosely related to a human proving a theorem. This problem is generally $NP$-hard and a direct approach may not be successful. However, building up to the result by first addressing a series of simpler but related training examples may provide enough insight to succeed.

The training samples $B_l$ are used to train a one-layer linear neural network to represent the solution $x$. In \cite{Welper2020}, a similar, but untrained, network is used as a relaxation technique to ease $\ell_p$ optimization. The extra layer introduces new parameters and is motivated by over-parametrization with additional regularization in neural network training, where empirical and theoretical evidence shows that this structure is beneficial for network training. See Section \ref{sec:compare-nn} for more details and references. Other examples that demonstrate how architectural choices can improve optimization are LSTM units \cite{HochreiterSchmidhuber1997} to improve vanishing/exploding gradients or skip connections \cite{HeZhangRenEtAl2016} for very deep networks.

In recent years, several connections have been established between compressed sensing and neural networks. The papers \cite{MardaniSunDonohoEtAl2018,ShiJiangZhangEtAl2017} implement the entire data to solution map $b \to x$ of compressed sensing by neural networks. In addition, there is a growing literature on solving under-determined linear systems $Ax = b$ with the prior assumption that the solution $x = G(z;w)$ is in the range of a generative neural network. The weights $w$ can be pre-trained on relevant data-sets and then the least squares loss $\min_z \|AG(z;w) - b\|_2$, or variants thereof, is minimized \cite{BoraJalalPriceEtAl2017,HandVoroninski2018,HuangHandHeckelEtAl2018,DharGroverErmon2018,WuRoscaLillicrap2019}. Alternatively, the papers \cite{VeenJalalSoltanolkotabiEtAl2020,JagatapHegde2019,HeckelSoltanolkotabi2020} use the deep image prior \cite{UlyanovVedaldiLempitsky2020} that $x$ is in the range of an untrained neural network and optimize $\min_w \|AG(z;w) - b\|_2$ for some latent variable $z$. These works replace sparsity by better prior information but keep common properties of the sensing matrix $A$, as i.i.d. Gaussian or variants of the Restricted Eigenvalue Condition (REC). Although using a similar structure, in this paper, we keep the sparsity prior but are interested in more general classes of sensing matrices for which no provable tractable sparse recovery algorithms seems to be available. 

The paper is organized as follows. In Section \ref{sec:setup}, we motivate the algorithms and consider some heuristic arguments for the sparse recovery results. In Section \ref{sec:compare-nn}, we compare the approach to neural networks and in Section \ref{sec:main}, we state the main recovery results, which are proven in Section \ref{sec:proof-main}.

\section{Problem Setup and Algorithms}
\label{sec:setup}

\paragraph{Data Model}

We want to solve a compressed sensing problem with input $b$, that is inaccessible to $\ell_1$-minimization or matching pursuit but supported by extra training samples $B_l$. These training samples must be somehow related to the main input $b$ to be helpful, which we establish by a minimalistic data model.

For a motivation, let us reconsider a human mathematics student from the introduction. The input $b$ may be compared to a final project, take-home exam or thesis. These are too advanced to be addressed at the beginning of a course and if assigned too early, the student would usually fail. Instead, she/he is first assigned homework problems, similar to $B_l$, which already contain components, techniques and tricks of the advanced problems but in a more isolated form or easier composition. Hence, they are intended to be solvable without much prior exposure to the subject and shall enable her/him to address the bigger problems later in her/his course of study.

We construct a crude analogy as follows: A proof corresponds to a solution vector $x \in \real^n$ of a compressed sensing problem with right hand side $b=Ax$ (both are $NP$-hard to find in the worst case). Each proof component can be considered as a proof in its own right, so we also represent them by column vectors $X_k \in \real^n$, combined into a \emph{component} matrix $X \in \real^{n \times p}$. In order to account for the idea that they are simpler or shorter than a full proof, we require them to be sparse. Finally, we combine these ``components'' into a full ``proof'' by a linear combination $x = X z$ with a \emph{combinator} vector $z \in \real^p$. For example a binary $z$ would amount to a (unordered) selection of components. This still has combinatorially many possibilities just as proving a big theorem form solid background knowledge is sill difficult. The training samples $B_l = XZ_l$ are generated by the same model but simpler in the sense that they can be addressed without prior knowledge by regular $\ell_1$ minimization, as discussed below.

Throughout this paper, the combinator vector $z$ is deterministic and $t/2$-sparse, allowing only limited length ``proofs'', and the component matrix $X$ is chosen randomly by a Bernoulli-Subgaussian model with expected sparsity $s$. 

\begin{definition}
  \label{def:bernoulli-subgaussian}

  We say that $X \in \real^{n \times p}$ satisfies the \emph{Bernoulli-Subgaussian model with parameter $s/n$} if $X_{jk} = \Omega_{jk} R_{jk}$, where $\Omega$ is an i.i.d. Bernoulli matrix and $R$ is an i.i.d. Subgaussian matrix with 
  \begin{align}
    \label{eq:bernoulli-subgaussian}
    \E{\Omega_{jk}} & = \frac{s}{n}, & 
    \E{R_{jk}} & = 0, & 
    \E{R_{jk}^2} & = \variance^2, & 
    \|R_{jk}\|_{\psi_2} & \le \variance K.
  \end{align}

\end{definition}

Recall that the $\psi_2$ norm is defined by $\|X\|_{\psi_2} := \sup_{p \ge 1} p^{-1/2} \E{|X|^p}^{1/p}$. For the main results, we need the following additional restrictions on this model.

\begin{definition}
  \label{def:restricted-bernoulli-subgaussian}

  We say that $X \in \real^{n \times p}$ satisfies the \emph{restricted Bernoulli-Subgaussian model with parameter $s/n$} if it satisfies the Bernoulli-Subgaussian model in Definition \ref{def:bernoulli-subgaussian}, with symmetric $R_{jk}$ and 
  \begin{align}
    \label{eq:restricted-bernoulli-subgaussian}
    \pr{R_{jk} = 0} & = 0, &
    \E{|R_{jk}|} & \in \left[ \frac{1}{10}, 1\right], &
    \E{R_{jk}^2} & \le 1, &
    \pr{|R_{jk}| > \tau} & \le 2 e^{ \frac{-\tau^2}{2}}.
  \end{align}

\end{definition}

Note that the last inequality implies the $\psi_2$-norm bound in \eqref{eq:bernoulli-subgaussian} with extra restrictions on the constant.

\paragraph{Sparse Recovery Given the Candidate Matrix}

Let us assume for the time being that we already know the component matrix $X$. Since we assume that the combinator $z$ is $t/2$-sparse, we search for it by the standard $\ell_1$-minimization problem 
\begin{equation}
  \begin{aligned}
    & \min_{z \in \real^p} \|z\|_1 & & \text{subject to} & AXz & = b,
  \end{aligned}
  \label{eq:csX}
\end{equation}
with modified sensing matrix $AX$. The extra matrix $X$ induces sufficient randomness to ensure that $AX$ satisfies the RIP with high probability, even if $A$ does not. This is proven in \cite{KasiviswanathanRudelson2019} and stated rigorously in Section \ref{sec:product-RIP} below. The recovery scheme is summarized in Algorithm \ref{alg:train} taking into account that we will not know the correct scaling of $X$, which is compensated by the diagonal scaling matrix
\begin{equation}
  \scale{X} := \operatorname{diag} \left( \|X_1\|_2^{-1}, \dots, \|X_p\|_2^{-1} \right).
  \label{eq:scaling-matrix}
\end{equation}
where again $X_k$ denotes the $k$-th column of $X$.

\begin{algorithm}
  \begin{algorithmic}
    \Function{SparseRecovery}{b, $\outX$}
      \State Let $\scale{\outX}$ be the scaling matrix defined in \eqref{eq:scaling-matrix}.
      \State Compute the solution $z$ of
      \begin{equation*}
        \begin{aligned}
	  & \min_z \|z\|_1 & & \text{subject to} & \frac{\sqrt{n}}{\|A\|_F} A \outX \scale{\outX} z = \frac{\sqrt{n}}{\|A\|_F} b
        \end{aligned}
      \end{equation*}
      \State \Return $\outX \scale{\outX} z$.
    \EndFunction
  \end{algorithmic}
  \caption{Sparse recovery with prior knowledge $\outX$.}
  \label{alg:sparse-recovery}
\end{algorithm}

\paragraph{Learning the Component Matrix}

We do not assume a-priori knowledge of the component matrix $X$, but rather recover it from training problems generated from the same data model $B_l = A X Z_l$ for $l = 1, \dots, q$ or short in matrix form $B = A X Z \in \real^{m \times q}$. In contrast to the right hand side $b$, we assume that sufficiently many of these problems are easy in the sense that they can be solved by $\ell_1$-minimization. In the student analogy above, these problems are easy enough for her to solve without prior training. 

The possibility of $\ell_1$ recovery for training problems may be by mere coincidence but can also be ensured systematically as follows: The solution $x = X z$ for right hand side $b$ is expected to be $st/2$-sparse. On the other hand, we may choose training samples $Z_l$ with more sparsity $\bar{t}/2 \ll t/2$ leading to $s\bar{t}/2$ sparse solutions $X Z_l$. These can be recovered for $s\bar{t}$-RIP matrices $A$, without necessarily being able to recover $x$ as well.

To recover $X$ form the training samples $B$, we first solve the $\ell_1$-minimization problems
\begin{equation*}
  \begin{aligned}
    & \min_{Y_l} \|Y_l\|_1 & & \text{subject to} & A Y_l = B_l
  \end{aligned}
\end{equation*}
and combine the columns $Y_l$ into a matrix $Y \in \real^{n \times q}$. Since we do not assume that every training sample is easy, we filter out the columns for which the $\ell_1$-recovery failed. To this end, we require one major assumption on the sensing matrix $A$: For some $u \ge st$ all $u$-sparse solutions of the system $Ax=B_l$ are unique, see e.g. \cite{FoucartRauhut2013}. Hence, if $Y_l$ is $st/2$ sparse, it must match the $st/2$ sparse vector $X Z_l$, providing an a-posteriori test for correctness. Filtering out the sparse columns of $Y$, we obtain
\[
  \outY := [Y_l : \,  \|Y_l\|_0 \le u, \, l = 1, \dots, q]
\]
with $\outY = X Z_L \in \real^{n \times \outq}$ for a corresponding sub-matrix $Z_L$ of $Z$. Now we can rigorously state what is meant by ``easy training samples'': The matrix $\outY$ must have rank at least $p$, i.e. sufficiently may independent samples are recovered correctly by $\ell_1$-minimization.

Note that we assume uniqueness of $u$-sparse solutions with $st \le u$ instead of simply $u = st/2$ to account for the randomness in the sparsity of $X$. This assumption requires sub-blocks with $2u$ columns to have full rank \cite{FoucartRauhut2013}, which is weaker than the corresponding $(\epsilon, 2u)$-RIP condition, which would ensure that the sparse solutions can be recovered by $\ell_1$-minimization and requires fairly strict bounds on the singular values of these sub-blocks.

Finally, we recover $X$ (and $Z_L$) from the known product $\outY = X Z_L$ by sparse matrix factorization. This is a well known problem in several contexts. For example in the transpose equation $\outY^T = Z_L^T X^T$, the columns of $Z_L^T$ can be interpreted as a dictionary, the rows of $X^T$ as random sparse coefficient vectors and the matrix $\outY$ as observations. The paper \cite{SpielmanWangWright2012} provides an algorithm to recover the dictionary $Z_L^T$ and coefficients $X^T$ up to scaling and permutation that implements the following method
\begin{definition}
  \label{def:sparse-factorization}
  Let $X \in \real^{n \times p}$ be restricted Bernoulli-Subgaussian \eqref{eq:restricted-bernoulli-subgaussian} and $Z \in \real^{p \times q}$, $p \le q$ be of full rank. Given $Y = XZ$, the algorithm 
  \[
    \bar{X},\,\bar{Z} = \SparseFactor(Y)
  \]
  returns two matrices $\bar{X} = X P \Gamma$ and $\bar{Z} = \Gamma^{-1} P^{-1} Z$, which match $X$ and $Z$ up to invertible diagonal scaling $\Gamma$ and permutation $P$, with probability at least $1 - C p^{-c}$ for some $c,C > 0$, independent of sparsity, dimensions and the probability model.
\end{definition}
Section \ref{sec:sparse-factorization} contains the algorithm, convergence guarantees and an overview over the literature. Algorithm \ref{alg:train} contains the full algorithm to recover the component matrix $X$ from samples $B$ and sparsity $u$.

\begin{remark}
  \label{remark:no-recovery-of-X}
  The exact reconstruction of $X$ up to scaling and permutation is convenient for the theory below, but not strictly necessary. For application in Algorithm \ref{alg:sparse-recovery}, we would merely need to find a matrix $\tilde{X}$ such that
  \begin{itemize}

    \item $A \tilde{X}$ satisfies the RIP.

    \item For every sparse $z$ there is a sparse $\tilde{z}$ with $Xz = \tilde{X}\tilde{z}$.
    
  \end{itemize}
  For example, keeping some sub-optimal $\ell_1$ reconstructed columns would be permissible as long as they don't disrupt the RIP property of $A \tilde{X}$.
\end{remark}

\begin{algorithm}
  \begin{algorithmic}
    \Function{Train}{$u$, $B$}
      \State For all $l = 1, \dots, q$, 
      \begin{equation*}
        \begin{aligned}
	  & Y_l = \argmin_y \|y\|_1 & & \text{subject to} & Ay & = B_l.
        \end{aligned}
      \end{equation*}
      \State Let $\outY \in \real^{p \times \outq}$ be the matrix with columns $Y_l$ for all $l$ with $\|Y_l\|_0 \le u$.
      \State \Return $\outX, \, \outZ = \SparseFactor(\outY)$
    \EndFunction
  \end{algorithmic}
  \caption{Training for maximal sparsity $u$ and samples $B \in \real^{m \times q}$.}
  \label{alg:train}
\end{algorithm}

\paragraph{Unknown Sparsity}

The above algorithm has the disadvantage that the sparsity $u$ must be provided as input. A crude remedy is to repeat the training Algorithm \ref{alg:train} and recovery Algorithm \ref{alg:sparse-recovery} for every possible $u \in \{1, \dots, n\}$. One of these reproduces the correct solution $x$ with high probability. By assumption, sparse solutions of $Ax = b$ are unique and therefore, the sparsest candidate must be the correct one. This crude approach adds a factor of $n$ to the overall runtime but still provides a polynomial time algorithm for a compressed sensing problem without a sufficiently strong RIP condition for $\ell_1$-recovery.

\begin{algorithm}
  \begin{algorithmic}
    \Function{TrainAndRecover}{$b$, $B$}
      \For{$u = 1,\dots,n$}
        \State $\outX, \outZ = \textalg{Train}(u, \, B)$
	\State $x_u = \textalg{SparseRecovery}(b, \outX)$
      \EndFor
    \EndFunction
    \State \Return $\argmin_{u = 1, \dots, n} \|x_u\|_0$
  \end{algorithmic}
  \caption{Sparse recovery for data $b$, given extra samples $B \in \real^{m \times q}$.}
  \label{alg:train-and-recover}
\end{algorithm}

\paragraph{Summary}

Compressed sensing algorithms for non-RIP matrices, or more accurately no null space property, typically rely on $\ell_p$-minimization with $0 < p < 1$, which more closely resembles the ideal $\ell_0$-norm, but renders the recovery problem non-convex. As a consequence, analytical guarantees for recovery algorithms typically require some form of locality to ensure convergence to the correct solution, although they may perform better in practice. Such a locality is not required for the analysis of the methods presented in this section. Instead, recovery guarantees for the polynomial time algorithms are based on ``easy'' and related training samples.

\section{Comparison with Neural Networks}
\label{sec:compare-nn}

In this section, we draw some parallels to neural network training. These have first lead to variants of Algorithm \ref{alg:sparse-recovery} in \cite{Welper2020} and are a major motivation for the results in this paper. To this end, let $h \in \real^{1 \times n}$ and $g \in \real^{1 \times m}$ be two consecutive hidden layers of a neural network, connected by a weight matrix $W \in \real^{n \times m}$ (including the bias for simplicity) and element-wise activation function $\sigma$:
\[
  g = \sigma(h W).
\]
We write the layers as row vectors to better highlight the similarities with the setup in this paper. In \eqref{eq:csX}, we split the solution $x$ of a compressed sensing problem into a product $Xz$ of a component matrix $X$ and a combinator $z$. In order to train the weight matrix $W$, we can apply the exact same strategy to each column of $W$, or equivalently to each neuron, $W_l = \cweights_l V_l$. In order to be more compliant with standard neural network architectures, we use a slight modification: Instead of assigning a component matrix for each neuron, let $\cweights \in \real^{n \times pm}$ be a wide matrix whose columns form a pool of components for all neurons of one layer. Then, for each neuron we use one combinator vector $V_l \in \real^{pm \times 1}$ to choose from this pool leading to the extended layer
\[
  g = \sigma(h \cweights V)
\]
with $V = [V_1, \dots, V_m]$. Effectively, we have ``over-parametrized'' the network by widening the layer $W$ to $\cweights$ and introducing a new linear layer $V$.

This setup is analyzed in \cite{Welper2020}, where $\cweights$ is interpreted as a collection of random initial guesses of a network optimization and $V$ is optimized to choose the best combination of the initial components by a relaxation argument. In an analogous compressed sensing problem, \cite{Welper2020} shows that this leads to an exponential speedup. Nonetheless, the overall performance of this approach is dominated by the negligible likelihood to contain components of the correct solution in the initial guess $\cweights$. This problem is overcome in this paper by learning the weights $\cweights$ from training samples, while keeping the idea for the $V$ optimization intact (with some slight simplifications).

Despite the comparison to compressed sensing, it is not clear if a simple added linear layer in a neural network does provide any meaningful effect on its optimization. For example, in \cite{AroraCohenGolowichEtAl2019} it is shown that the optimization of fully linear networks is comparable to the optimization of least squares problems with all layers squashed together. Let us therefore consider some more parallels between the two areas.

\begin{itemize}

  \item In the compressed sensing problem, the split-up of the solutions of $[b,B] = A [x, Y]$ into the matrix product $[x,Y] = X [z,Z]$ induces some redundancy or ``over-parametrization'', which is counterbalanced with the requirement that both $X$ and $[z,Z]$ are sparse. A similar setup is often found in neural networks: They are also over-parametrized and then regularized to compensate, which aids training both empirically \cite{GoodfellowVinyals2015,ZhangBengioHardtEtAl2017} and theoretically \cite{SoudryCarmon2016,SafranShamir2018,LiLiang2018,Allen-ZhuLiSong2019,DuLeeLiEtAl2019,AroraDuHuEtAl2019,OymakSoltanolkotabi2020,Allen-ZhuLi2020}. For the regularization, $\ell_1$ penalties are one option \cite{GoodfellowBengioCourville2016,{WenWuWangEtAl2016}} and also more popular dropout regularization tends to promote sparsity \cite{SrivastavaHintonKrizhevskyEtAl2014,MolchanovAshukhaVetrov2017}. In addition, sparse weight matrices are generated by pruning in network compression, see e.g. \cite{HanPoolTranEtAl2015,HanMaoDally2015} or \cite{Neill2020} for an overview. Finally, our added layer is missing non-linear activation functions, which can also deactivate entire neurons, e.g., an added $\operatorname{ReLU}$ unit for our extra layer $\sigma( \operatorname{ReLU}(h \cweights) V)$ is equivalent to setting entries of $V$ to zero in $h$ and $\cweights$ dependent locations.

  \item The training protocol for neural networks is different from compressed sensing. The component matrix $X$ corresponds to $\cweights$ and the combinators $[z,\outZ]$ combined correspond to $V$,  all trained by one single gradient descent optimization. In particular, note that $z$ computed by Algorithm \ref{alg:sparse-recovery} corresponds to a network weight or neuron and not a forward run of the network. If an analogous gradient descent training with $\ell_1$ penalties would work for our layered compressed sensing setup is unknown, but Remark \ref{remark:no-recovery-of-X} indicates that the algorithmic framework of this paper is likely more strict than necessary.

  \item As for neural networks, the training samples are presented to the algorithms in bulk, without any separation in ``simple'' and ``hard'' problems. If there is an analogue of ``simple'' problems for neural networks is unknown, but good quality training data is important for successful neural network training, as well, see e.g. the discussion in \cite{Allen-ZhuLi2020}. 

  \item Given the component matrix $X$, there exists an efficient algorithm (Algorithm \ref{alg:sparse-recovery}) to determine the remaining vector $z$ to solve compressed sensing problems, with high probability. If it is possible to successfully train $V$ given $\cweights$ is much less clear, but some loose analogies do exist. The matrix $\cweights$ is very wide and random, either from initialization, or training on data sets generated by a random data model, corresponding to $X$ and $\tilde{X}$ in Remark \ref{remark:no-recovery-of-X}. Such ingredients are used, e.g., in \cite{Allen-ZhuLiSong2019,DuLeeLiEtAl2019} to show convergence of gradient descent to global minima, although in a different setup.

\end{itemize}

To what extend the above analogies hold up to scrutiny is left for future research. Anyways, the provided compressed sensing setup is a potential toy model to explain some aspects of neural network training. Conversely, neural networks have motivated the simple data model in Section \ref{sec:setup} and may provide ideas for more complex ones in the future.

\section{Main Result}
\label{sec:main}

In this section, we prove the observations we made during the construction of the algorithms. Throughout this article, $c$ and $C$ are non-negative generic constants that can change in every term but are independent of the dimensions, sparsity and the Bernoulli-Gaussian model.

We first verify that Algorithm \ref{alg:sparse-recovery} recovers the correct combinator vector $z$ given the component matrix $\outX$ provided by Algorithm \ref{alg:train}. 

\begin{proposition}
  \label{prop:sparse-recovery}
  Assume that $x$ is generated by the data model $x = X z$ for some $t/2$-sparse combinator vector $z\in \real^p$ and component matrix $X$ for which $\frac{\sqrt{n}}{\|A\|_F} A X S_X$ is a $(t, \epsilon)-RIP$ matrix, with scaling $S_X$ defined in \eqref{eq:scaling-matrix} and $\epsilon < 4/\sqrt{41} \approx 0.6246$. Let $b = Ax$ and $\outX = XP\Gamma$ be a scaled permutation of $X$ for some permutation $P$ and signed scaling matrix $\Gamma$ (i.e. the output of Algorithm \ref{alg:train}). Then, $\textalg{SparseRecovery}(b, \outX) = X z = x$, i.e. Algorithm \ref{alg:sparse-recovery} reproduces $x$.
\end{proposition}

\begin{proof}

We have $x = Xz = \outX \scale{\outX} \outz$ with $\outz := \scale{\outX}^{-1}  \Gamma^{-1} P^{-1} z$ and thus
\[
  \frac{\sqrt{n}}{\|A\|_F} A \outX \scale{\outX} \outz
  = \frac{\sqrt{n}}{\|A\|_F} A x
  = \frac{\sqrt{n}}{\|A\|_F} b.
\]
Since $\outz$ is a rescaling and permutation of $z$, it is $t/2$-sparse and since $\frac{\sqrt{n}}{\|A\|_F} A \outX \scale{\outX}$ satisfies the $(t, \epsilon)$-RIP, it is the unique solution of the $\ell_1$-minimization problem in Algorithm \ref{alg:sparse-recovery}, see e.g. \cite{FoucartRauhut2013}. It follows that the algorithm returns $\outX \scale{\outX} \outz = X z = x$.

\end{proof}

In order to invoke the previous proposition, we need the scaled permutation $\outX$ of the component matrix. This is provided by Algorithm \ref{alg:train} by the following theorem. It uses the stable rank of a matrix defined by
\[
  \srank(A) := \frac{\|A\|_F^2}{\|A\|^2},
\]
which is a variant of the rank with more stable behaviour for small singular values.

\begin{theorem}
  \label{th:train}

  For Algorithm \ref{alg:train} assume that
  \begin{enumerate}[label=(A\arabic*)]

    \item The input data is generated by $B = AXZ$, with a matrix $A \in \real^{m \times n}$, a restricted $s/n$-Bernoulli Subgaussian component matrix $X \in \real^{n \times p}$, as in Definition \ref{def:restricted-bernoulli-subgaussian}, and deterministic combinators $Z \in \real^{p \times q}$ with $t/2$-sparse columns.

    \item  We have
    \begin{align}
      p & \le q, &
      n & > c_1 p^2 \log^2 p, & 
      \frac{2}{p} & \le \frac{s}{n} \le \frac{c_2}{\sqrt{p}}, &
      \label{eq:assumption:sparse-factorization}
    \end{align}
    \begin{equation}
      \srank(A) \ge C K^4 \frac{n t}{s \epsilon^2} \log \left( \frac{3p}{\epsilon t} \right)
      \label{eq:assumption:rip}
    \end{equation}
    for some constants $c_1, c_2, C \ge 0$ and $\psi_2$-norm bound $K$ from the Bernoulli-Subgaussian model \eqref{eq:bernoulli-subgaussian}. The first assumption \eqref{eq:assumption:sparse-factorization} ensures sparse matrix factorization and the second \eqref{eq:assumption:rip} the RIP of $AX$.

    \item \label{assumption:sparse-l1-recovery} The input $u$ of the algorithm satisfies $u \ge st$ and for every $b \in \real^m$, the system $Ax = b$ has at most one $u$-sparse solution.
    
    \item \label{assumption:enough-simple} For the output $\outX$, $\outZ$ of Algorithm \ref{alg:train}, the matrix $\outY = \outX \outZ$ has at least $\outq \ge p$ columns and rank $p$.

  \end{enumerate}

  Then there are constants $c > 0$ and $C \ge 0$ independent of the probability model, dimensions and sparsity so that with probability at least
  \[
    1 - C p^{-c}
  \]
  the output $\outX$ of Algorithm \ref{alg:train} is a scaled permutation permutation $\outX = X P \Gamma$ of the component matrix $X$ with permutation $P$ and invertible signed scaling matrix $\Gamma$ and $A\outX$ satisfies the RIP
  \begin{equation}
    (1-\epsilon) \|v\|_2
    \le \frac{\sqrt{n}}{\|A\|_F}\|A \outX \scale{\outX} v \|_2
    \le (1+\epsilon) \|v\|_2
    \label{eq:th:train:RIP}
  \end{equation}
  for all $t$-sparse vectors $v \in \real^p$.
  
\end{theorem}

The assumptions \eqref{eq:assumption:sparse-factorization} \eqref{eq:assumption:rip} and Assumption \ref{assumption:sparse-l1-recovery} pose restrictions on the sizes and sparsity of the involved matrices and vectors. Since they are not independent of each other, we have to ensure that all of them can be met at once. We do so first with a rather crude heuristic, ignoring all log factors. A more thorough argument is given in Section \ref{sec:assumptions} below. In the following, we optimize the sizes so that $A$ becomes as flat as possible, assuming that $A$ has full stable rank $\srank(A) \approx m$.

Let us first consider the RIP condition \ref{eq:assumption:rip}. We choose the RIP constant $\epsilon$ small enough so that $\ell_1$-recovery is possible, but not too close to zero so that $\epsilon \approx 1$. Ignoring log factors, and making the stable rank as small as possible, the RIP condition \eqref{eq:assumption:rip} then reduces to $\srank(A) \approx \frac{n}{s} t$. For a flat $A$ with few rows or small stable rank, it is beneficial if the matrix $X$ has low sparsity and high randomness, which is ensured by a large $s$. 

This has to be balanced with the requirement of small sparsity $s$ for the unique recovery condition in Assumption \ref{assumption:sparse-l1-recovery}. Indeed, unique recovery of $u \approx st$-sparse vectors requires at least $2u \approx 2st \le m \approx \srank(A)$ rows. Again, choosing the smallest feasible stable rank, together with the RIP constraints from above, we have $st \approx \srank(A) \approx \frac{n}{s} t$ or equivalently $s^2 \approx n$.

Finally, we have to satisfy the sparse factorization assumptions \eqref{eq:assumption:sparse-factorization}. Plugging $s^2 \approx n$ into the second and third inequalities and ignoring log factors, we obtain $s^2 \ge p^2$ and $\frac{2}{p} \lesssim \frac{1}{s} \lesssim \frac{c_2}{\sqrt{p}}$. The first two inequalities then imply that the assumption is satisfied with $p \approx s$.

So how few rows can $A$ have? Using $s^2 \approx n$ and the sparse recovery assumption $st \approx \srank(A) \approx m$ from above, we obtain $t \sqrt{n} \approx \srank(A) \approx{m}$. Thus, unlike regular compressed sensing, the matrix $A$ cannot have exponentially more columns than rows. However, choosing $n$ large and creating some extra space for the sparsity $t$, the ratio of rows and columns can still be arbitrarily small.

The remaining Assumption \ref{assumption:enough-simple} depends on the output of the algorithm and can therefore only be verified a-posteriori. It reflects our earlier discussion that the samples $B$ contain sufficiently many easy cases so that after filtering out unsuccessful $\ell_1$ recoveries, the remaining matrix $\outY = \outX \outZ$ has sufficiently many independent columns.

Finally, we show the correctness of Algorithm \ref{alg:train-and-recover}, which lifts the dependence on the unknown sparsity bound $u$.

\begin{theorem}
  \label{th:train-and-recover}
  Assume that all assumptions of Theorem \ref{th:train} hold for unknown $u$ and $\epsilon < 4/\sqrt{41} \approx 0.6246$, with data $B = XZ$. Let $x$ be generated by the data model $x = X z$ for some $t/2$-sparse vector $z\in \real^p$ and $b = Ax$. 

  Then there are constants $c > 0$ and $C \ge 0$ independent of the probability model and dimensions so that with probability at least
  \[
    1 - C p^{-c}
  \]
  the output of Algorithm \ref{alg:train-and-recover} yields $\textalg{TrainAndRecover}(b, B) = x$.
\end{theorem}

In summary, for carefully chosen dimensions and sparsity levels, the algorithms work as expected in the motivation: Given a training sample $B$ that contains sufficiently many simple examples, we can indeed recover the component matrix $X$ and then also solve compressed sensing problems generated by our data model for which no alternative algorithms with provable recovery guarantees are known.

\section{Proof of the Main Results}
\label{sec:proof-main}

In this section, we prove the main results of this paper. First, we show RIP properties of $AX$ in Section \ref{sec:product-RIP}, then implementations and results for the component $\SparseFactor$ in Section \ref{sec:sparse-factorization}. They are combined to the main results in Section \ref{sec:proof-algorithms} and finally a discussion of the assumptions is provided in Section \ref{sec:assumptions}.

\subsection{The RIP for \texorpdfstring{$AX$}{AX}}
\label{sec:product-RIP}

The RIP of $AX$ follows directly from \cite{KasiviswanathanRudelson2019}. In this section, we show some corollaries because our algorithms use a different scaling of $X$.

\begin{theorem}[{\cite[Theorem 3.4]{KasiviswanathanRudelson2019}}]
  \label{th:RIP-AX}
  Let $A \in \real^{m \times n}$ be a matrix and $X \in \real^{n \times p}$ satisfy the Bernoulli-Subgaussian-model \eqref{eq:bernoulli-subgaussian}.
  For RIP constant $0 < \epsilon < 1$ and sparsity $t \in \mathbb{N}$ satisfying
  \begin{equation}
    \label{eq:RIP-AX-condition}
    \srank(A) \ge C K^4 \frac{n t}{s \epsilon^2} \log \left( \frac{p}{\epsilon t} \right)
  \end{equation}
  with probability at least
  \[
    1 - \exp \left( - \frac{c}{K^4} \frac{s}{n} \epsilon^2 \srank(A) \right)
  \]
  the matrix $AX$ satisfies the RIP
  \[
    (1-\epsilon) \|v\|_2
    \le \sqrt{\frac{n}{s}} \frac{1}{\variance \|A\|_F}\|A X v \|_2
    \le (1+\epsilon) \|v\|_2
  \]
  for all $t$-sparse vectors $v \in \real^p$.
\end{theorem}

The columns of $X$ in the last theorem are scaled by their expected $\ell_2$-norms. In Algorithm \ref{alg:train}, we only recover $X$ up to scaling and the normalize all columns by their actual $\ell_2$-norms with the matrix $\scale{X}$. The following Corollary provides the RIP in this scenario.

\begin{corollary}
  \label{cor:RIP-AX}
  Let $A \in \real^{m \times n}$ be a matrix and $X \in \real^{n \times p}$ satisfy the Bernoulli-Subgaussian-model \eqref{eq:bernoulli-subgaussian}.
  For RIP constant $0 < \epsilon < 1$ and sparsity $t \in \mathbb{N}$ satisfying
  \begin{equation}
  \label{eq:RIP-scale-condition}
    \srank(A) \ge C K^4 \frac{n t}{s \epsilon^2} \log \left( \frac{3p}{\epsilon t} \right)
  \end{equation}
  with probability at least
  \[
    1 - 3 \exp \left( - \frac{c}{K^4} \frac{s}{n} \epsilon^2 \srank(A) \right)
  \]
  the matrix $AX$ satisfies the RIP
  \[
    (1-\epsilon) \|v\|_2
    \le  \frac{\sqrt{n}}{\|A\|_F}\|A X \scale{X} v \|_2
    \le (1+\epsilon) \|v\|_2
  \]
  for all $t$-sparse vectors $v \in \real^p$.
\end{corollary}

The prove relies on the simple observation that with high probability the actual scaling in $\scale{X}$ and the expected scaling $\sqrt{s}$ are close. This can be shown by several concentration estimates. We use the following Hanson-Wright inequality.
\cite[Theorem C.5]{KasiviswanathanRudelson2019}

\begin{theorem}[{\cite[Theorem 1.1]{Zhou2019}}] 
  \label{th:sparse-hanson-wright}
  Let $M \in \real^{n \times n}$ be a matrix and $X \in \real^{n \times 1}$ satisfy the Bernoulli-Subgaussian-model \eqref{eq:bernoulli-subgaussian}. Then, for every $\tau \ge 0$
  \[
    \pr{ \left| X^T M X - \E{X^T M X} \right| \ge \variance^2 \tau } \le 2 \exp \left( -c \min \left\{ \frac{n \tau^2}{s K^4 \|M\|_F^2}, \, \frac{\tau}{K^2 \|M\|} \right\} \right)
  \]
\end{theorem}

\begin{proof}
  Note that the cited theorem uses the convention $\|R_{jk}\|_{\psi_2} \le K$, whereas we use $\|R_{jk}\|_{\psi_2} \le \variance K$. This extra $\variance$ is cancelled by the lower bound $\variance^2 \tau$ in the left hand side. The rest is identical.
\end{proof}

The concentration estimate applied to our situation and removing some squares in the probability on the left hand side yields the following lemma. The arguments are similar to the proof of Theorem \ref{th:RIP-AX} in \cite{KasiviswanathanRudelson2019}.

\begin{lemma}
  \label{lemma:norm-concentration}
  Let $X \in \real^{n \times 1}$ satisfy the Bernoulli-Subgaussian-model \eqref{eq:bernoulli-subgaussian}. Then, for every $\tau \ge 0$
  \[
    \pr{ \left| \|X\|_2 - \sqrt{s}\variance  \right| \ge \sqrt{s} \variance\tau } \le 2 \exp \left( -\frac{c}{K^4} s \tau^2 \right).
  \]
\end{lemma}

\begin{proof}

  We first apply Theorem \ref{th:sparse-hanson-wright} to find a concentration estimate for $\|X\|_2^2$. To this end, let $M = I$ be the identity matrix. Then, we have the identities 
  \begin{align*}
    X^T M X & = \|X\|_2^2, &
    \|M\|_F^2 & = n, & 
    \|M\| & = 1
  \end{align*}
  and since $\Omega_j^2 = \Omega_j$ the expectation
  \[
    \E{X^T M X} = \sum_{j=1}^n \E{\Omega_j^2} \E{R_j^2} = n \frac{s}{n} \variance^2 = s \variance^2.
  \]
  Applying Theorem \ref{th:sparse-hanson-wright} implies
  \begin{equation}
    \begin{aligned}
      \pr{ \left| \|X\|_2^2 - s\variance^2 \right| \ge s \variance^2 \theta } 
      & \le 2 \exp \left( -c s \min \left\{ \frac{\theta^2}{K^4}, \, \frac{\theta}{K^2} \right\} \right) \\
      & \le 2 \exp \left( -\frac{c}{K^4} s \min \left\{ \theta^2, \, \theta \right\} \right),
    \end{aligned}
    \label{eq:norm-square-concentration}
  \end{equation}
  where in the last step we have used that $K^2 \le c K^4$ because by definition of the $\|\cdot\|_{\psi_2}$-norm and $K$ in \eqref{eq:bernoulli-subgaussian}, we have $K \ge 2^{-1/2}$.

  In the next step, we use an argument form \cite[Lemma C.6]{KasiviswanathanRudelson2019} to find a concentration estimate for $\|X\|_2$ instead of its square. To this end, we show that for $\tau^2 = \min\{\theta^2, \theta\}$, we have
  \begin{equation}
    \begin{aligned}
      \left| \|X\|_2^2 - s\variance^2 \right| & \le s \variance^2 \theta &
      & \Rightarrow & 
      \left| \|X\|_2 - \sqrt{s}\variance \right| & \le \sqrt{s}\variance \tau,
    \end{aligned}
    \label{eq:remove-square}
  \end{equation}
  which together with the concentration inequality \eqref{eq:norm-square-concentration} directly yields the lemma.

  Let us assume that the left inequality of \eqref{eq:remove-square} or equivalently $\left| \frac{1}{s\variance^2} \|X\|_2^2 - 1 \right| \le \theta$ holds. Using that $|r-1| \le |r^2-1|$ for $r \ge 0$ and $\theta = \max\{\tau, \tau^2\}$, in the case $\theta = \tau$, we obtain the rescaled inequality
  \[
    \left| \frac{1}{\sqrt{s}\variance} \|X\|_2 - 1 \right|
    \le \left| \frac{1}{s\variance^2} \|X\|_2^2 - 1 \right| \le \theta = \tau.
  \]
  Likewise, with $|r-1|^2 \le |r^2-1|$ for $r \ge 0$, in the case $\theta = \tau^2$, we have 
  \[
    \left| \frac{1}{\sqrt{s}\variance}\|X\|_2 - 1 \right|^2
    \le \left| \frac{1}{s\variance^2}\|X\|_2^2 - 1 \right| \le \theta = \tau^2
  \]
  Multiplying the last two inequalities with $\sqrt{s}\variance$ and $s\variance^2$, respectively directly yields \eqref{eq:remove-square} and the lemma.
  
\end{proof}

As a corollary, we obtain the scaling behaviour of $\scale{X}$.

\begin{corollary}
  \label{cor:scale-norm}
  Let $X \in \real^{n \times p}$ satisfy the Bernoulli-Subgaussian-model \eqref{eq:bernoulli-subgaussian}. Then, for every $\tau \ge 0$ satisfying $s \ge C K^4 \frac{1}{\tau^2} \log p$ with probability at least $1 - 2 \exp \left( - \frac{c}{K^4} s \tau^2 \right)$ we have
  \[
    (1-\tau) \|v\| \le \frac{1}{\sqrt{s}\variance} \left\|\scale{X}^{-1} v \right\| \le (1+\tau) \|v\|
  \]
\end{corollary}

\begin{proof}

By assumption, we have $C \log p \le \frac{1}{K^4} s \tau^2$ for sufficiently large $C$. From Lemma \ref{lemma:norm-concentration} with a union bound, it follows that with probability at least 
\[
  1- 2 p \exp \left( -\frac{\bar{c}}{K^4} s \tau^2 \right)
  = 1- 2 \exp \left( -\frac{\bar{c}}{K^4} s \tau^2 - \log p\right)
  \le 1- 2 \exp \left( -\frac{c}{K^4} s \tau^2 \right)
\] 
we have
\begin{equation*}
  \begin{aligned}
    \left| \|X_k\| - \sqrt{s}\variance \right| & \le \sqrt{s}\variance \tau &
    & \text{for all} &
    k & = 1, \dots, p.
  \end{aligned}
\end{equation*}
This implies
\begin{equation*}
  \begin{aligned}
    1 - \tau \le \frac{1}{\sqrt{s}\variance}\|X_k\| & \le 1 + \tau & 
    & \text{for all} &
    k & = 1, \dots, p
  \end{aligned}
\end{equation*}
and therefore, by definition of the diagonal scaling matrix $S_X$, we have
\begin{equation*}
  \begin{aligned}
    1 - \tau \le \frac{1}{\sqrt{s}\variance} \left(S_X^{-1}\right)_{kk} & \le 1 + \tau & 
    & \text{for all} &
    k & = 1, \dots, p.
  \end{aligned}
\end{equation*}
which directly implies the corollary.

\end{proof}

We now have all ingredients to prove the RIP under scaling by $\scale{X}$.

\begin{proof}[Proof of Corollary \ref{cor:RIP-AX}]

The result directly follows from the RIP of $AX$ in Theorem  \ref{th:RIP-AX} with RIP constant $0 \le \tmpepsilon := \epsilon/3 \le 1/3$ and the scaling result Corollary \ref{cor:scale-norm} with $\tau = \epsilon/3$. 

Let us first verify the assumptions. For sufficiently enlarged constants $c$ and $C$, the given RIP condition \eqref{eq:RIP-scale-condition} is identical to the RIP condition \eqref{eq:RIP-AX-condition} with modified RIP constant $\tmpepsilon = \epsilon/3$ so that the latter is clearly satisfied. Since $\|A\|_F \le \sqrt{n} \|A\|$ we have $\srank(A) \le n$ and since $\tmpepsilon \le 1/3$ and $t\le p$ by Lemma \ref{lemma:log-inequality} we have $t \log \left( \frac{p}{\tmpepsilon t} \right) \ge t \log \left( \frac{3p}{t} \right) \ge \log p$. It follows that
\begin{equation*}
  n \ge \srank(A) 
  \stackrel{\eqref{eq:RIP-scale-condition}}{\ge } C K^4 \frac{n t}{s \tmpepsilon^2} \log \left( \frac{p}{\tmpepsilon t} \right)
  \ge C K^4 \frac{n}{s \tmpepsilon^2} \log p,
\end{equation*}
which upon cancelling $n$ and defining $\tau = \tmpepsilon$ directly implies the condition of Corollary \ref{cor:scale-norm} for sufficiently large generic constant $C$.

Using $\frac{\srank(A)}{n} \le 1$ in the success probability of Corollary \ref{cor:scale-norm}, with probability at least
\[
  1 - 3 \exp \left( - \frac{c}{K^4} \frac{s}{n} \tmpepsilon^2 \srank(A) \right)
\]
both the RIP
\[
  (1-\tmpepsilon) \sqrt{s} \variance \|v\|_2
  \le \frac{\sqrt{n}}{\|A\|_F}\|A X v \|_2
  \le (1+\tmpepsilon) \sqrt{s} \variance \|v\|_2
\]
and scaling bounds
\[
  \frac{1}{1+\tmpepsilon} \|v\| \le 
  \sqrt{s}\variance \|\scale{X} v\| 
  \le \frac{1}{1-\tmpepsilon} \|v\| 
\]
are satisfied. It follows that
\begin{multline*}
  \frac{1-\tmpepsilon}{1+\tmpepsilon} \|v\|_2
  \le (1-\tmpepsilon) \sqrt{s} \variance \|\scale{X} v\|_2
  \\
  \le \frac{\sqrt{n}}{\|A\|_F}\|A X \scale{X} v \|_2
  \\
  \le (1+\tmpepsilon)\sqrt{s} \variance \|\scale{X} v\|_2
  \le \frac{1+\tmpepsilon}{1-\tmpepsilon} \|v\|_2
\end{multline*}
For $0 \le \tmpepsilon \le 1/3$, one easily verifies that
\begin{align*}
  \frac{1 - \tmpepsilon}{1 + \tmpepsilon} & \ge 1 - 3 \tmpepsilon, &
  \frac{1 + \tmpepsilon}{1 - \tmpepsilon} & \le 1 + 3 \tmpepsilon, &
\end{align*}
which directly implies the result with $\epsilon = 3 \tmpepsilon$.

\end{proof}

\subsection{Sparse Matrix Factorization}
\label{sec:sparse-factorization}

In order to implement \SparseFactor{} from Definition \ref{def:sparse-factorization}, we use results from sparse dictionary learning: Given a dictionary (or basis) in the columns of $Z^T$ and random sparse coefficient vectors in the columns of $X^T$, the task is to recover the dictionary from observations $Y^T = Z^T X^T$. In our algorithm, we solve the equivalent transpose problem to recover $X$ and $Z$ from observations $Y = XZ$. The decomposition $XZ$ is invariant under permutations and scaling, which are already accounted for in the definition of \SparseFactor{}. The following result is from \cite{SpielmanWangWright2012}.

\begin{theorem}[{\cite[Theorem 3, Theorem 9]{SpielmanWangWright2012}}]
  \label{th:sparse-factorization}
  Let $Z \in \real^{p \times q}$, with $p \le q$ be a full rank matrix, $X \in \real^{n \times p}$ be a restricted Bernoulli-Subgaussian matrix \eqref{eq:restricted-bernoulli-subgaussian} with parameter $s/n$ and
  \begin{align}
    n & > c_1 p^2 \log^2 p, &
    \frac{2}{p} & \le \frac{s}{n} \le \frac{c_2}{\sqrt{p}}.
    \label{eq:sparse-factorization-assumptions}
  \end{align}
  Then Algorithm \ref{alg:sparse-factorization} provides a tractable implementation of \SparseFactor, with success probability at least $1 - C n^{-c}$ for constants $c > 0$ and $C \ge 0$.
\end{theorem}

\begin{remark}
  In contrast to the last theorem, in the definition of \SparseFactor{} we only requires the smaller success probability of $1 - C p^{-c}$ because that is sufficient for the proof of theorem \ref{th:train}.
\end{remark}

\begin{algorithm}
  \begin{algorithmic}
    \Function{ER-SpUD(DC)}{$Y$}
      \State Randomly pair the rows of $Y$ into $n/2$ groups with indices $j = (j_1, j_2)$
      \For{$j = 1,\dots,n/2$}
        \State $w = \argmin_{w} \|Yw\|_1$ subject to $(e_{j_1} + e_{j_2})^T Y w = 1$
        \State $s_j := Y w$
      \EndFor
      \State \Return $s_j$, $j=1,\dots,n/2$
    \EndFunction
  \end{algorithmic}
  
  \begin{algorithmic}
    \Function{Greedy}{$Y$}
      \State  $s =$ \Call{ER-SpUD(DC)}{$Y$}
      \While{$i \le \operatorname{rank}([s_1, \dots, s_n])$}
        \State $x_i = \argmin_j \|s_j\|_0$ such that $[x_1, \dots, x_i]$ has full rank.
      \EndWhile
    \EndFunction
  \end{algorithmic}
  \caption{\SparseFactor}
  \label{alg:sparse-factorization}
\end{algorithm}

The result is a combination of two theorems in \cite{SpielmanWangWright2012} and the main conclusion of this paper. For the convenience of the reader, we include a short proof directly form these two theorems.

\begin{proof}
  We use \cite[Theorem 3 and Theorem 9]{SpielmanWangWright2012} with the following change of notations: $p, n, \frac{s}{n}$ in our case correspond to $n, p, \theta$ in the reference.

  Let us first consider the case that $Z$ is square. Then, by the given assumptions and \cite[Theorem 9]{SpielmanWangWright2012}, the output $s_1, \dots, s_{n/2}$ of \textalg{ER-SpUD(DC)}($Y$) contains the columns of $X$ with probability at least $1-C n^{-10}$. In the next step, we greedily select the sparsest linear independent $s_i$ to build the returned matrix $\bar{X}$. This ensures that $\max_k \|\bar{X}_k\|_0 \le \max_k \|X_k\|_0$. In addition, by construction $\bar{X}$ has full column rank and the same number of rows as $Y$, so that it determines a unique $\bar{Z}$ with $Y = \bar{X} \bar{Z}$. By \cite[Theorem 3]{SpielmanWangWright2012}, such a sparse decomposition is unique up to permutation and scaling, with probability at least
  \[
    1-C p\exp^{-c s} \ge 1 - C \exp^{-c \log n},
  \]
  so that the returned matrices satisfy the requirements of \SparseFactor. The last inequality can be easily derived from assumption \eqref{eq:sparse-factorization-assumptions} by $s \ge 2 \frac{\sqrt{n} \sqrt{n}}{p} \ge 2 \sqrt{c_1} \sqrt{n} \log p \ge c \sqrt{n} \ge c \log n$. Note that the $\ell_0$ optimization in the algorithm is only carried out over a discrete set of candidates and is therefore tractable. In conclusion, for square $Z$, Algorithm \ref{alg:sparse-factorization} provides an implementation of \SparseFactor.

  Non-square matrices are not considered in \cite{SpielmanWangWright2012}. However, the algorithm only depends on the column span of $Y$, which equals the column span of $X$ if $Z$ is square and invertible or alternatively if $p \le q$ and $Z$ has full rank. Therefore, the results are applicable unchanged.

\end{proof}

\begin{corollary}
  \label{cor:full-rank}
  Let $X$ satisfy all assumptions in Theorem \ref{th:sparse-factorization}. Then, with probability at least $1 - Cn^{-c}$ the matrix $X$ has full column rank.
\end{corollary}

\begin{proof}

  As shown in the proof of Theorem \ref{th:sparse-factorization}, with the given assumptions and probabilities, by \cite[Theorem 3]{SpielmanWangWright2012}, the sparse factorization $Y = XZ$ is unique, for square and invertible $Z$ i.e. for any $Y=\tilde{X}\tilde{Z}$, with $\|\tilde{X}_k\|_0 \le \|X\|_0$ for all columns $k=1, \dots, p$, we have $\tilde{X} = XS^{-1}P^{-1}$ and $\tilde{Z} = PSZ$ for scaling matrix $S$ and permutation $P$. 
  
  Assume that $X$ does not have full column rank and choose $Z = I$. Then there is a non-zero vector $v \in \real^p$ such that $X v = 0$ and with the matrix $V = [v, \dots, v] \in \real^{p \times q}$, $\alpha < \|v\|_\infty$ and $\tilde{Z} := Z + \alpha V$ we have $X \tilde{Z} = X + \alpha XV = XZ = Y$. Hence, by uniqueness, there must be scaling and permutation matrices such that $Z + \alpha V = PSZ$, or equivalently $I - PS = \alpha V$, using that $Z=I$. The subtraction $PS$ on the left hand side can change at most on entry per column and since the diagonal entries of $\alpha V$ are unequal to one, by the choice $\alpha < \|v\|_\infty$, it must be the diagonal entries. But then the left hand side of $I-PS = \alpha V$ cannot consist of identical non-zero columns as the right hand side contracting the uniqueness result.
  
\end{proof}

Some of the earlier works on \emph{dictionary learning} or \emph{sparse coding} provide uniqueness results \cite{AharonEladBruckstein2006} and local optimization properties \cite{GribonvalSchnass2010,Schnass2015}. The results from \cite{SpielmanWangWright2012} used above, seem to be the first polynomial time algorithms with provable dictionary recovery guarantees. They fit well within this paper, but newer results are available.

A first major direction of improvement is to ease the upper bound in the second assumption in \eqref{eq:sparse-factorization-assumptions} and allow less sparse coefficient vectors $X$. These include optimization over spheres by trust-region methods \cite{SunQuWright2017,SunQuWright2017a}, $\ell_4$-norm maximization \cite{ZhaiYangLiaoEtAl2020}, combinations of clustering and alternating minimization \cite{AgarwalAnandkumarJainEtAl2014}, \cite{AroraGeMoitra2014,AroraBhaskaraGeEtAl2014,AroraGeMaEtAl2015} and tensor decompositions \cite{BarakKelnerSteurer2015}.  In our context, less sparse matrices $X$ would be beneficial to show RIP conditions but have to be balanced with the unique sparse recovery Assumption \ref{assumption:sparse-l1-recovery} leading to comparatively high sparsity requirements independent of the dictionary recovery. 

Newer results as in \cite{AgarwalAnandkumarJainEtAl2014,AroraGeMoitra2014,AroraBhaskaraGeEtAl2014,AroraGeMaEtAl2015,BarakKelnerSteurer2015} also allow $Z^T$ to be over-complete with less rows than columns $q \le p$, usually with some extra incoherence assumptions. In our application, this translates to fewer samples and is left for future work.

Finally, newer methods are more resilient to added noise and related to neural networks \cite{NeyshaburPanigrahy2014,AroraGeMaEtAl2015}. In particular, the latter discusses the neural plausibility of their algorithms, which is directly relevant for the discussion in Section \ref{sec:compare-nn}.

\subsection{Proof of the Algorithms}
\label{sec:proof-algorithms}

In this section, we prove the main results from Section \ref{sec:main}. The recovery of the component matrix $X$ relies on the assumption that $u$-sparse solutions of $Ax = b$ are unique, although probably not computable by $\ell_1$-minimization. To ensure this sparsity for $Xz$ with random $X$ and sparse $z$, we first provide a corresponding concentration estimate.

\begin{lemma}
  \label{lemma:bernoulli-subgaussian-sparsity}
  Let $X \in \real^{n \times p}$ be a Bernoulli-Subgaussian matrix with parameter $s/n$ and $\alpha \ge 0$. Then
  \[
    \pr{\exists k \in \{1, \dots, p\}: \, \|X_k\|_0 \ge (1+\alpha) s}
    \le 2 \exp \left( - \frac{3 \alpha^2}{6 + 2\alpha} s + \log p \right).
  \]
  
\end{lemma}

\begin{proof}

For a single column $k \in \{1, \dots, p\}$, we have
\begin{multline*}
  \pr{\|X_k\|_0 \ge (1+\alpha) s}
  \le \pr{\|\Omega_k\|_0 \ge (1+\alpha) s}
  \\
  = \pr{\sum_{j=1}^n \left( \Omega_{jk} - \frac{s}{n} \right) \ge \alpha s}
  \le 2 \exp \left( - \frac{\alpha^2 s^2/2}{n (s/n) + \alpha s/3}\right)
  = 2 \exp \left( - \frac{3 \alpha^2}{6 + 2 \alpha} s\right)
\end{multline*}
by Bernstein's inequality using that $\left| \Omega_{jk} - \frac{s}{n} \right| \le 1$, $\E{\Omega_{jk} - \frac{s}{n}} = 0$ and $\E{\left(\Omega_{jk} - \frac{s}{n} \right)^2} \le \frac{s}{n}$ because $\Omega_{jk}^2 = \Omega_{jk}$. Applying a union bound over all columns yields the lemma.

\end{proof}

\begin{proof}[Proof of Theorem \ref{th:train}]

Let $L \subset \{1, \dots, q\}$ be the columns of $B$ selected for the definition of the matrix $\outY$ in Algorithm \ref{alg:train}. We first show that the matrix $\outY$, recovered by $\ell_1$ minimization, is a sub-matrix of the data model $XZ$. By the second and third inequalities in \eqref{eq:assumption:sparse-factorization}, we have $s \ge \frac{2n}{p} \ge 2c_1 p \log^2 p$, so that by Lemma \ref{lemma:bernoulli-subgaussian-sparsity}, with probability at least
\begin{equation}
  1 - 2 \exp \left( - c p \right)
  \label{eq:proof:p-sparse}
\end{equation}
all columns of $X$ are $2s$-sparse. It follows that all columns $XZ_L$ are $2st/2 \le u$-sparse solutions of the equation $B_L := AXZ_L$. Likewise, by the selection criterion, all columns of $\outY$ are $u$-sparse solutions of the same system $B_L = A \outY$. Since $u$-sparse solutions of this system are unique by Assumption \ref{assumption:sparse-l1-recovery}, it follows that $\outY = X Z_L$.

By Corollary \ref{cor:full-rank} with probability at least \eqref{eq:proof:p-sparse} the matrix $X$ has rank $p$ and by Assumption \ref{assumption:enough-simple} the matrix $\outY = X Z_L$ has rank $p$. It follows that $Z_L$ must have full rank. Hence, together with the assumptions \eqref{eq:assumption:sparse-factorization}, we can invoke Theorem \ref{th:sparse-factorization} so that with probability at least
\begin{equation}
  1 - C p^{-c} = 1 - C \exp( -c \log p)
  \label{eq:proof:factorization}
\end{equation}
the step $\outX, \outZ = \SparseFactor(Y)$ recovers the two factors $X$ and $Z_L$ up to permutation $P$ and signed scaling $\Gamma$, i.e.
\begin{align*}
  \outX & = X P \Gamma, & 
  \outZ & = \Gamma^{-1} P^{-1} Z_L.
\end{align*}
This shows the first part of the theorem.

Next, we show the RIP condition. By assumption \eqref{eq:assumption:rip}, we can apply Corollary \ref{cor:RIP-AX} so that with probability at least
\begin{equation}
  1 - 3 \exp \left( - \frac{c}{K^4} \frac{s}{n} \epsilon^2 \srank(A) \right)
  \label{eq:proof:RIP}
\end{equation}
we have the $(u, \epsilon)$-RIP
\[
  (1-\epsilon) \|v\|_2
  \le \frac{\sqrt{n}}{\|A\|_F}\|A X \scale{X} v \|_2
  \le (1+\epsilon) \|v\|_2
\]
for all $t$-sparse $v$. Since $X \scale{X} = \outX \scale{\outX} \sign(\Gamma)^{-1} P^{-1}$, by Lemma \ref{lemma:rescale} below, with element-wise $\sign(\cdot)$, it follows that
\[
  (1-\epsilon) \|v\|_2
  \le \frac{\sqrt{n}}{\|A\|_F}\|A \outX \scale{\outX} \sign(\Gamma)^{-1} P^{-1} v \|_2
  \le (1+\epsilon) \|v\|_2
\]
This directly implies the RIP condition \eqref{eq:th:train:RIP} because the $\ell_2$-norm is invariant with respect to $P$ and $\sign(\Gamma)$.

Finally, we add up the success probabilities. From assumption \eqref{eq:assumption:rip}, we have $\frac{c}{K^4} \frac{s}{n} \epsilon^2 \srank(A) \ge c C t \log \left( \frac{3p}{\epsilon t} \right)$ and therefore, the RIP probability \eqref{eq:proof:RIP} simplifies to
\begin{equation*}
  1 - 3 \exp \left( - \frac{c}{K^4} \frac{s}{n} \epsilon^2 \srank(A) \right)
  \ge 1 - 3 \exp \left( c C t \log \left( \frac{3p}{\epsilon t} \right) \right)
  \ge 1 - 3 \exp \left( c C \log p \right),
\end{equation*}
where in the last inequality we have used that $t \log \left( \frac{3p}{\epsilon t} \right) \ge t \log \left( \frac{3p}{t} \right) \ge \log p$ for $0 \le \epsilon \le 1$ and $1 \le t \le p$ by Lemma \ref{lemma:log-inequality}). In combination with the other probabilities \eqref{eq:proof:p-sparse}, \eqref{eq:proof:factorization}, this yields a success probability of at least
\[
  1 - C \exp \left( - c \log p \right) = 1 - C p^{-c}
\]
for some modified generic constants $c$ and $C$.

\end{proof}

\begin{proof}[Proof of Theorem \ref{th:train-and-recover}]

The algorithm loops over all $u = 1, \dots, n$. We first consider the choice that matches $u$ in Assumption \ref{assumption:sparse-l1-recovery}. Then, by Theorem \ref{th:train} and Proposition \ref{prop:sparse-recovery} the vector $x_u$ computed in the algorithm satisfies $x_u = x$, with probability at least $1 - C p^{-c}$.

It remains to show that Algorithm \ref{alg:train-and-recover} returns this vector and not some other $x_{u'}$ with  $u' \ne u$. If $x_u$ is $u$-sparse this directly follows from the unique sparse recovery Assumption \ref{assumption:sparse-l1-recovery} because for every solution $Ax_{u'} = b$ we either have $x_{u'} = x_u$ or $\|x_{u'}\|_0 > \|x_u\|_0$. 

In order to show that $x_u$ is $u$-sparse with high probability, we repeat the argument form \eqref{eq:proof:p-sparse}: By the second and third inequalities in \eqref{eq:assumption:sparse-factorization}, we have $s \ge \frac{2n}{p} \ge 2c_1 p \log^2 p \ge c \log p$, so that by Lemma \ref{lemma:bernoulli-subgaussian-sparsity}, with probability at least
\begin{equation}
  1 - 2 \exp \left( - c p \right)
\end{equation}
all columns of $X$ are $2s$-sparse. Hence $x$ is a $2s t/2 = st \le u$-sparse solution of $Ax = b$. Adding up all probabilities and eventually redefining the generic constants $c$ and $C$ concludes the proof.

\end{proof}

\subsection{Feasibility of the Assumptions}
\label{sec:assumptions}

Throughout this section $\lesssim, \sim , \gtrsim$ denote smaller, equivalent, and larger up to some generic constants independent of the dimensions, sparsity and probabilistic models.

\paragraph{The a-priori Assumptions}

The assumptions of Theorem \ref{th:train} are not fully independent and we have to verify that they can all be satisfied. A heuristic argument, ignoring log factors is provided after the theorem, but a closer examination shows that it cannot be correct. Indeed, using the choice $s^2 \sim n$ from this discussion to eliminate $s$ in the sparse factorization assumption \eqref{eq:assumption:sparse-factorization}, we obtain $\frac{p}{2} \gtrsim \sqrt{n} \gtrsim p \log p$, which is impossible for large $p$. Instead, we make the two relevant inequalities in \eqref{eq:assumption:sparse-factorization} sharp and choose
\begin{align}
  p & \sim \frac{n}{s}, &
  n & \sim p^2 \log^2 p,
  \label{eq:choice-s-p-n}
\end{align}
which implies our original choice up to an added log factor and ensures all assumptions in  \eqref{eq:assumption:sparse-factorization}. The number of samples $q$ is independent of the other choices and can easily be made sufficiently large $q \ge p$.

It remains to ensure that the sensing matrix $A$ satisfies the unique sparse recovery property in Assumption \ref{assumption:sparse-l1-recovery} and the stable rank bounds in \eqref{eq:assumption:rip}. As we will see, both are implied by a RIP condition satisfied by i.i.d. Gaussian matrices with high probability. Since this allows $\ell_1$ recovery, it is not necessarily the envisioned application, but good enough to demonstrate that the assumptions of Theorem \ref{th:train} are feasible.

To show the RIP, we first choose the remaining number of rows $m$ of $A$ and sparsity of the combinators $t$ such that 
\begin{equation}
  m \ge C(K) \frac{t}{\epsilon^2} \sqrt{n} \log^2 n
  \label{eq:n-rows}
\end{equation}
for some sufficiently large constant $C(K)$ that may depend on $K$. As in our heuristic motivation, this allows $n$ to be roughly $m^2$ up to some safety margin for $t$ and log factors so that $A$ can be arbitrarily flat but not exponentially as for $\ell_1$ recovery. 

In order to show the RIP, by applying $\log$ to $n > c p^2 \log^2 p$ we obtain
\[
  \log n > \log p^2 + \log \log^2 p + \log c \ge \log p 
\]
for $\log p \ge \max\{1, |\log c|\}$ and thus
\begin{equation}
  st 
  \sim t \frac{n}{p} 
  \sim t \frac{p^2 \log^2 p}{p} 
  = t \left( p \log p \right) \log p 
  \lesssim t \sqrt{n} \log p 
  \lesssim t \sqrt{n} \log n 
\end{equation}
Hence, using that $\frac{n}{s} \le c p$ and $\log \left( \frac{ce p}{2t} \right) \le \log p \lesssim \log n$, with the extra harmless technical assumption $\frac{ce}{2t} \le 1$, we obtain
\[
  C(K) \epsilon^{-2} s t \log \left( \frac{en}{2st} \right) 
  \lesssim C(K) \epsilon^{-2} \left[ t \sqrt{n} \log n \right] \log \left( \frac{cep}{2t} \right) 
  \lesssim C(K) \epsilon^{-2} t \sqrt{n} \log^2 n
  \le m
\]
so that by \cite[Theorem 9.2]{FoucartRauhut2013} the matrix $A$ satisfies a $(2st, \epsilon)$-RIP with high probability. This directly implies the weaker condition of unique $st$-sparse recovery in Assumption \ref{assumption:sparse-l1-recovery}.

Finally, we verify the lower bound on the stable rank in assumption \eqref{eq:assumption:rip}. Since the singular values in each $st$ wide sub-block of a $(\epsilon, st)$-RIP matrix are close to one, it is easy to verify that the RIP implies $\srank(A) \gtrsim st$, see Lemma \ref{lemma:sr-from-RIP} below, so that it suffices to estimate the right hand side. By \eqref{eq:choice-s-p-n}, we have
\begin{equation*}
  \begin{aligned}
    s \sim \frac{n}{p} & \sim p \log^2 p &
    & \Leftrightarrow &
    p \log p \sim \frac{s}{\log p}
  \end{aligned}
\end{equation*}
and thus, again assuming that $t$ is not close to one so that $\frac{3}{\epsilon t} \le 1$, we have
\begin{align*}
  C(K) \frac{t}{\epsilon^2} \frac{n}{s} \log \left( \frac{3p}{\epsilon t} \right)
  \lesssim \frac{t}{\epsilon^2} p \log p
  \lesssim \frac{1}{\epsilon^2} \frac{1}{\log p} s t
  \lesssim  \srank(A)
\end{align*}
for $\log p$ larger than the involved fixed constants and the last inequality follows from $st \lesssim \srank(A)$. This directly implies the RIP assumption \eqref{eq:assumption:rip}. 

In conclusion, we can satisfy all required assumptions for sensing matrices $A$ with at least $m$ rows given by \eqref{eq:n-rows} or up to log factors $m \gtrsim t \sqrt{n} $. Although the feasibility has been demonstrated for Gaussian matrices, this is not necessarily the intended application because these would allow a much simpler $\ell_1$-recovery.

\paragraph{The a-posteriori Assumption}

The a-posteriori Assumption \ref{assumption:enough-simple} requires that sufficiently many $\ell_1$-optimization problems
\begin{equation*}
  \begin{aligned}
    & \min_y \|y\|_1 & & \text{subject to} & A y = B_l 
  \end{aligned}
\end{equation*}
for $l = 1 \dots, q$ in the training set yield $u$-sparse solutions. The existence of such sparse solutions is induced by our data generation process: $B_l$ is defined as $A(X Z_l)$ for $t/2$-sparse vectors $Z_l$ and $s/n$-Bernoulli Subgaussian matrix $X$, which has $2s$-sparse columns with high probability by Lemma \ref{lemma:bernoulli-subgaussian-sparsity}. Hence $XZ_l$ are a $2st/2 \le u$-sparse and satisfy the constraints.

We now have to verify that sufficiently many of these sparse solutions $X Z_l$ are recovered by $\ell_1$-minimization. In order to provide an example, assume that $A$ allows the recovery of $s\bar{t}/2 < st/2$ sparse vectors but not necessarily of $st/2$-sparse vectors. We generate our training data with a deterministic matrix $Z$ that contains a set of $\bar{t}/2$-sparse linearly independent columns. With high probability, these will be selected in Algorithm \ref{alg:train} and thus be contained in the returned matrix $\outZ$, up to scaling and permutation, so that the a-posteriori Assumption \ref{assumption:enough-simple} is satisfied.

Although the training data $B = A(XZ)$ contains sufficiently many ``easy'' problems, solvable by $\ell_1$-minimization, we can recover $x = Xz$ for all $t/2$-sparse vectors $z$ by Algorithm \ref{alg:sparse-recovery} after training. These $x$ are expected to be $st/2$-sparse, which is not necessarily solvable by $\ell_1$-minimization in our setup.

\subsection{Technical Lemmas}

This section contains some technical lemmas that are used in the proofs above.

\begin{lemma}
  \label{lemma:log-inequality}
  Let $1 \le t \le p$. Then, we have
  \[
    t \log \frac{3p}{t} \ge \log p.
  \]
\end{lemma}

\begin{proof}

Define $f(t) := t \log \frac{3p}{t}$. Then, we have $f(1) = \log(3p) \ge \log p$ and $f$ is monotonically increasing:
\[
  f'(t) = \log \frac{3p}{t} - 1 \ge \log(3) - 1 > 0,
\]
where we have used that $t \le p$.

\end{proof}

\begin{lemma}
  \label{lemma:rescale}
  Let $\outX = X P \Gamma$ for permutation matrix $P$ and invertible signed scaling matrix $\Gamma$. Then
  \[
    X \scale{X} = \outX \scale{\outX} \sign(\Gamma)^{-1} P^{-1}.
  \]
\end{lemma}

\begin{proof}

Assume that $P e_i = e_{p(i)}$. Then the $i$-th column of $\outX$ is given by
\[
  \outX e_i 
  = X P \Gamma e_i 
  = X P e_i \gamma_i 
  = X e_{p(i)} \gamma_i 
  = X_{p(i)} \gamma_i.
\]
It follows that the $i$-th column of $\outX \scale{\outX}$ is
\[
  \outX \scale{\outX} e_i
  = \frac{1}{\|\gamma _i X_{p(i)}\|} \gamma_i X_{p(i)} 
  = \frac{1}{\|X_{p(i)}\|} X_{p(i)} \sign(\gamma_i).
\]
Likewise, we have
\[
  X \scale{X} P \sign(\Gamma) e_i
  = X \scale{X} e_{p(i)} \sign(\gamma_i)
  = \frac{1}{\|X_{p(i)}\|} X_{p(i)} \sign(\gamma_i).
\]
Comparing the last two equations, we have $\outX \scale{\outX} = X \scale{X} P \sign(\Gamma)$, which directly proves the lemma.

\end{proof}

\begin{lemma}
  \label{lemma:sr-from-RIP}
  Assume that $A \in \real^{m \times n}$ satisfies a $(s, \epsilon)$-RIP. Then 
  \[
    \srank(A) \ge \left( \frac{1}{2} \frac{1-\epsilon}{1+\epsilon} \right) s.
  \]
\end{lemma}

\begin{proof}

Let $T_i$, $i=0,\dots,r$ be a partition of $\{1, \dots, n\}$ with $|T_i| = s$ for all $i$, except the first $|T_0| \le s$ and $A_T$ be a sub-matrix with columns in $T$. Then, we have
\begin{align*}
  \|A\| 
  & = \sup_x \frac{\|Ax\|}{\|x\|} 
  \le \sup_x \frac{\sum_{i=0}^r \|A_{T_i}\| \|x_{T_i}\|}{\|x\|} 
  \le \max_{i = 0, \dots, r} \|A_{T_i}\| \sup_x \frac{\sum_{i=0}^r \|x_{T_i}\|}{\|x\|} 
  \\
  & \le \max_{i = 0, \dots, r} \|A_{T_i}\| \sup_x \frac{\sqrt{r+1} \left(\sum_{i=0}^r \|x_{T_i}\|^2 \right)^{1/2}}{\|x\|} 
  = \sqrt{r+1}\max_{i = 0, \dots, r} \|A_{T_i}\|
  \\
  & \le \sqrt{r+1}\max_{|T| \le s} \|A_T\|
\end{align*}
and
\[
  \|A\|_F^2
  \ge \sum_{i=1}^r \|A_{T_i}\|^2_F
  \ge r \min_{i=1, \dots, r} \|A_{T_i}\|^2_F
  \ge r \min_{|T| = s} \|A_T\|^2_F,
\]
where in the first step we have dropped the set $T_0$, which is not of the right size. From the last two estimates, we have
\[
  \srank(A)
  = \frac{\|A\|_F^2}{\|A\|^2}
  \ge \frac{r}{r+1}\frac{\min_{|T| = s} \|A_T\|_F^2}{\max_{|T| = s} \|A_T\|^2}
  \ge \frac{1}{2} \frac{\min_{|T| = s} \|A_T\|_F^2}{\max_{|T| = s} \|A_T\|^2}.
\]
By the $(s,\epsilon)$-RIP, for every $T$ with $|T|=s$, all singular values $\sigma$ of $A_T$ are bounded by $(1-\epsilon) \le \sigma \le (1+\epsilon)$ and therefore
\[
  \srank(A) \ge \frac{1}{2} \frac{1-\epsilon}{1+\epsilon} s.
\]

\end{proof}

\bibliographystyle{abbrv}
\bibliography{cs-train}

\end{document}